\documentclass[letterpaper]{article} 
\usepackage{aaai2026}  
\usepackage{times}  
\usepackage{helvet}  
\usepackage{courier}  
\usepackage[hyphens]{url}  
\usepackage{graphicx} 
\usepackage{natbib}  
\usepackage{caption} 
\usepackage{algorithm}
\usepackage{algorithmic}

\usepackage{newfloat}
\usepackage{listings}
\DeclareCaptionStyle{ruled}{labelfont=normalfont,labelsep=colon,strut=off} 
\floatstyle{ruled}
\newfloat{listing}{tb}{lst}{}
\floatname{listing}{Listing}
\usepackage[utf8]{inputenc} 
\usepackage[T1]{fontenc}    
\usepackage{url}            
\usepackage{booktabs}       
\usepackage{amsfonts}       
\usepackage{nicefrac}       
\usepackage{microtype}      
\usepackage{amssymb}        
\usepackage{amsmath}        
\usepackage{amsthm}         
\usepackage{bigdelim}       
\usepackage[inline,shortlabels]{enumitem}
\usepackage{xspace}
\usepackage{cleveref}       
\usepackage{acro}           
\usepackage{threeparttable}
\usepackage{multirow}
\usepackage{makecell}
\usepackage{colortbl}
\usepackage{pifont}         
\usepackage[table, dvipsnames]{xcolor}  

\newcounter{corrfn}

\def\corrauthor{%
  \ifnum\value{corrfn}=0%
    \footnote{Corresponding author.}%
    \setcounter{corrfn}{\value{footnote}}%
  \else%
    \footnotemark[\value{corrfn}]%
  \fi%
}%

\title{Learning Neural Operators from Partial Observations \\ 
        via Latent Autoregressive Modeling}
\author{
    Jingren Hou\equalcontrib\textsuperscript{\rm 1 \rm 3},
    Hong Wang\equalcontrib\textsuperscript{\rm 2},
    Pengyu Xu\corrauthor\textsuperscript{\rm 1 \rm 3},
    Chang Gao\textsuperscript{\rm 1 \rm 3},
    Huafeng Liu\textsuperscript{\rm 1 \rm 3},
    Liping Jing\corrauthor\textsuperscript{\rm 1 \rm 3}
}
\affiliations{
    \textsuperscript{\rm 1}School of Computer Science and Technology, Beijing Jiaotong University, Beijing, China\\
    \textsuperscript{\rm 2}University of Science and Technology of China\\
    \textsuperscript{\rm 3}Beijing Key Laboratory of Traffic Data Mining and Embodied Intelligence, Beijing, China

    \{21112019, pengyuxu, 22110098, hfliu1, lpjing\}@bjtu.edu.cn, wanghong1700@mail.ustc.edu.cn
}

\newcommand{\ours}{\textsc{LANO}\xspace}
\theoremstyle{plain}
\newtheorem{theorem}{Theorem}[section]

\theoremstyle{definition}

\theoremstyle{remark}

\begin{document}

    \maketitle
    

\begin{abstract}
Real-world scientific applications frequently encounter incomplete observational data due to sensor limitations, geographic constraints, or measurement costs. 
Although neural operators significantly advanced PDE solving in terms of computational efficiency and accuracy, their underlying assumption of fully-observed spatial inputs severely restricts applicability in real-world applications.
We introduce the first systematic framework for learning neural operators from partial observation.
We identify and formalize two fundamental obstacles: (i) the supervision gap in unobserved regions that prevents effective learning of physical correlations, and (ii) the dynamic spatial mismatch between incomplete inputs and complete solution fields.  
Specifically, our proposed Latent Autoregressive Neural Operator~(\ours) introduces two novel components designed explicitly to address the core difficulties of partial observations:
(i) a mask-to-predict training strategy that creates artificial supervision by strategically masking observed regions, and (ii) a Physics-Aware Latent Propagator that reconstructs solutions through boundary-first autoregressive generation in latent space.
Additionally, we develop POBench-PDE, a dedicated and comprehensive benchmark designed specifically for evaluating neural operators under partial observation conditions across three PDE-governed tasks.
\ours achieves state-of-the-art performance with 
18--69$\%$ relative L2 error reduction 
across all benchmarks under patch-wise missingness with 
less than 50$\%$ missing rate,
including real-world climate prediction. 
Our approach effectively addresses practical scenarios 
involving up to 75$\%$ missing rate,
to some extent bridging the existing gap between idealized research settings and the complexities of real-world scientific computing.
\end{abstract}

\begin{links}
\link{Code}{https://github.com/Kingyum-Hou/LANO}
\end{links}





\section{Introduction}
Neural operators provide a promising data-driven alternative for solving partial differential equations (PDEs) in science and engineering~\cite{zachmanoglou1986introduction}. 
They learn to approximate the input-output mappings of PDE-governed tasks from data during training and then infer the PDE solutions for accurate simulations~\cite{karniadakis2021physics,li2020fourier, gao_pinn}. Compared to traditional PDE solvers, neural operators can infer solutions over five orders of magnitude faster during inference~\cite{azizzadenesheli2024neural, wangaccelerating, dong2024accelerating, wang2025symmap}.
Beyond computational efficiency, neural operators successfully model a wide range of physical phenomena governed by underlying PDEs, including airflow, weather patterns, and optical systems~\cite{li2023fourier,gupta2022towards,brandstetter2022clifford, luo2024neural, huang2025self, wang2025stnet, lv2025exploiting, wei2025mecot, wei2025vflow}. 


\begin{figure}[!t]
    \centering
    \begin{minipage}{0.99\linewidth}
    \centerline{\includegraphics[width=\textwidth]{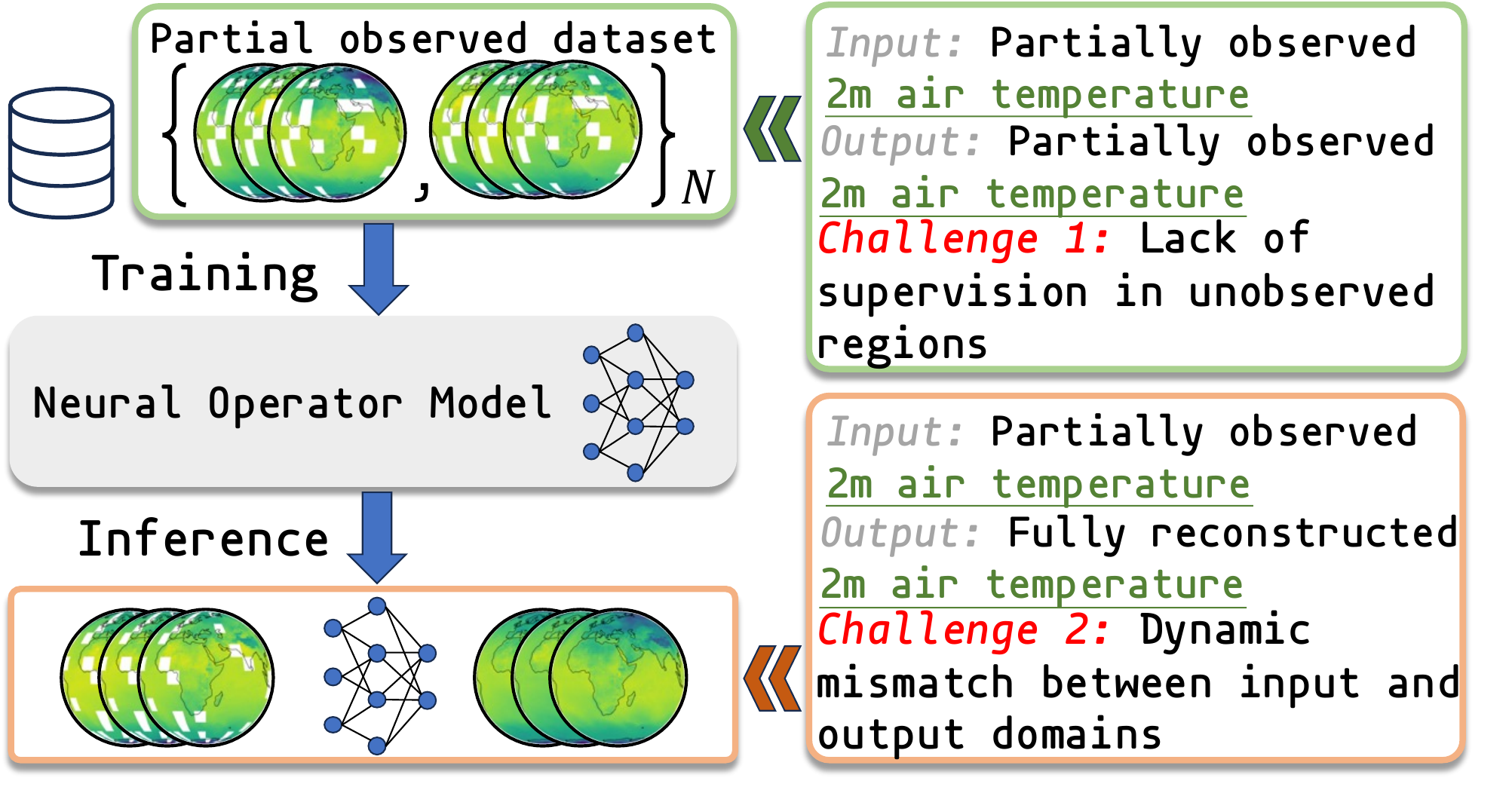}}
    \end{minipage}
    \caption{An overview of neural operator learning under partially observed PDE datasets. Once trained, the model is expected to infer solutions from unseen, partially observed inputs. However, two challenges arise in this setting, as summarized on the right.}
    \label{fig:intro}
\end{figure} 

A major obstacle for neural operators is their dependency on high-quality and fully observed training datasets~\cite{kovachki2023neural}.
However, in many real-world applications across the natural sciences and engineering, data are often partially observed—for instance, weather records miss large geographic areas where stations cannot be installed~\cite{morice2021updated}. 
Similar gaps arise in seismic exploration, electromagnetic surveys, and medical imaging~\cite{zheng2019applications,puzyrev2019deep,feng2021task} due to terrain blockage, sparse sampling, or sensor failure.
\textbf{Crucially, each training sample normally consists of an input field and its corresponding solution field on a fixed grid; Partial observation removes spatially correlated unobserved regions from both.}
These extensive and irregular unobserved regions make interpolation-based data completion unreliable for training neural operators. 

This partial observability poses two challenges for operators learning (Fig.~\ref{fig:intro}):
(i) Lack of supervision in unobserved regions:
In partial observations, the solution fields are only partially available.
The absence of ground truth in unobserved regions hinders the model from learning intricate physical correlations of PDEs ~\cite{wu2024transolver}.
(ii) Dynamic mismatch between partial input and global output domains: Unlike classical neural operator settings where inputs and outputs are defined over the same spatial grid~\cite{brandstettermessage2022,wu2024transolver}, partially observed data results in mismatched and dynamically varying input regions.
The model must infer full-domain outputs from partially observed inputs, which most existing architectures are not designed to handle~\cite{gao2025discretizationinvariance}.

\paragraph{Our approach} 
To address these challenges, we propose a hybrid solution that combines a novel masking-based training strategy with a physics-aware autoregressive modeling framework, and we investigate: \textit{can masked prediction tasks within observed regions implicitly guide operator networks to accurately infer unobserved regions?}.

Firstly, we introduce a mask-to-predict training strategy~(MPT). Inspired by the masking strategy in natural language processing~(NLP) and computer vision~(CV), which is conceptually simple: it removes a portion of the data and learns to predict the removed content~\cite{devlin2019bert, he2022masked}, MPT artificially masks parts of the already observed input regions during supervised training.
This creates pseudo-missing regions with available supervision, encouraging the model to extrapolate from partial context.
Instead of pre-training on unlabeled data of neural operators~\cite{chen2024data}, MPT is directly integrated into the training process, enabling better generalization to real unobserved regions.

Secondly, we propose the latent autoregressive neural operator~(LANO). Directly predicting all unobserved regions simultaneously often results in blurry and incoherent solutions due to the inherent difficulty of modeling complex joint distributions over spatially dependent points. 
This issue has been similarly noted in visual content generation, where autoregressive or parallelized prediction strategies have effectively mitigated global modeling difficulties by explicitly handling local dependencies~\cite{van2016pixel, wang2025parallelized}.
Motivated by this observation, LANO employs an autoregressive generation strategy within a latent representation space, progressively reconstructing unobserved regions in a spatially structured manner.
Central to this design is the Physics-aware Latent Propagator (PhLP), which decomposes solution prediction into a structured multi-step process. Starting from observed boundary conditions, PhLP incrementally propagates latent information into unobserved regions. This boundary-first design reflects the physical structure of many PDEs and enforces physical consistency in predictions.

To our knowledge, this is the first neural operator framework that enables effective training from partially observed data with dynamic and unaligned input–output domains.
Our main contributions are summarized as follows:

\begin{itemize}
   \item We identify and formalize two key challenges in operator learning under partial observation: (i) the lack of supervision in unobserved regions, and (ii) the dynamic mismatch between partial input and global output domains
   \item We propose a hybrid framework to address these challenges: (i) the Mask-to-Predict (MPT) training strategy for robust learning from partial observation, and (ii) the Latent Autoregressive Neural Operator (\ours), which uses a physics-aware latent propagator (PhLP) to progressively reconstruct solutions across unobserved regions.
   \item We construct POBench-PDE, the first benchmark suite for operator learning under partial observation with six sets of comprehensive experiments across three PDE-governed tasks, and demonstrate that under patch-wise missingness with less than $50\%$ missing rate, our model consistently outperforms existing methods by $17.8\%$-$68.7\%$ in relative error reduction on both synthetic and real-world PDE-governed tasks, including weather forecasting.
\end{itemize}



\section{Related Work}
\label{sec:related}

\paragraph{Neural Operator Learning.}  
This field aims to approximate PDE input-output mappings through deep models. DeepONet~\cite{lu2021learning} pioneered this direction, followed by Fourier Neural Operator (FNO)~\cite{li2020fourier} and its variants~\cite{li2024physics, tran2023factorized, wen2022u}. To handle irregular meshes, Geo-FNO~\cite{li2023fourier} introduces coordinate mappings for non-uniform domains.

Transformer-based approaches have gained prominence, including GK-Transformer~\cite{Cao2021transformer}, OFormer~\cite{li2022transformer}, GNOT~\cite{hao2023gnot}, ONO~\cite{xiao2023improved}, and MoE-POT~\cite{wangmixture}. Recent methods explore latent space modeling: Transolver~\cite{wu2024transolver} uses Physics-Attention for geometric features, LSM~\cite{wu2023solving} operates in hierarchical latent space, while IPOT~\cite{lee2024inducing}, LNO~\cite{wang2024latent}, and UPT~\cite{alkin2024universal} decouple input-output sampling locations.

However, these methods assume shared input-output discretization or produce one-shot full-domain predictions, limiting their effectiveness under partial observation. In contrast, \ours enables progressive information propagation from observed to unobserved regions through boundary-first latent evolution.

\paragraph{Training under Sparse Observation.} 
Some transformer-based methods (OFormer, IPOT, GNOT, LNO) handle partial inputs during inference but require fully observed training, limiting real-world application. DINo~\cite{sitzmann2020implicit} and CORAL~\cite{serrano2023operator} employ implicit neural representations for continuous prediction from sparse signals.

These approaches either rely on implicit modeling or lack systematic training under dynamic sparse patterns, focusing primarily on sparse point-wise missingness rather than spatially correlated unobserved regions.


\begin{figure*}[t]
    \centerline{\includegraphics[width=0.99\linewidth]{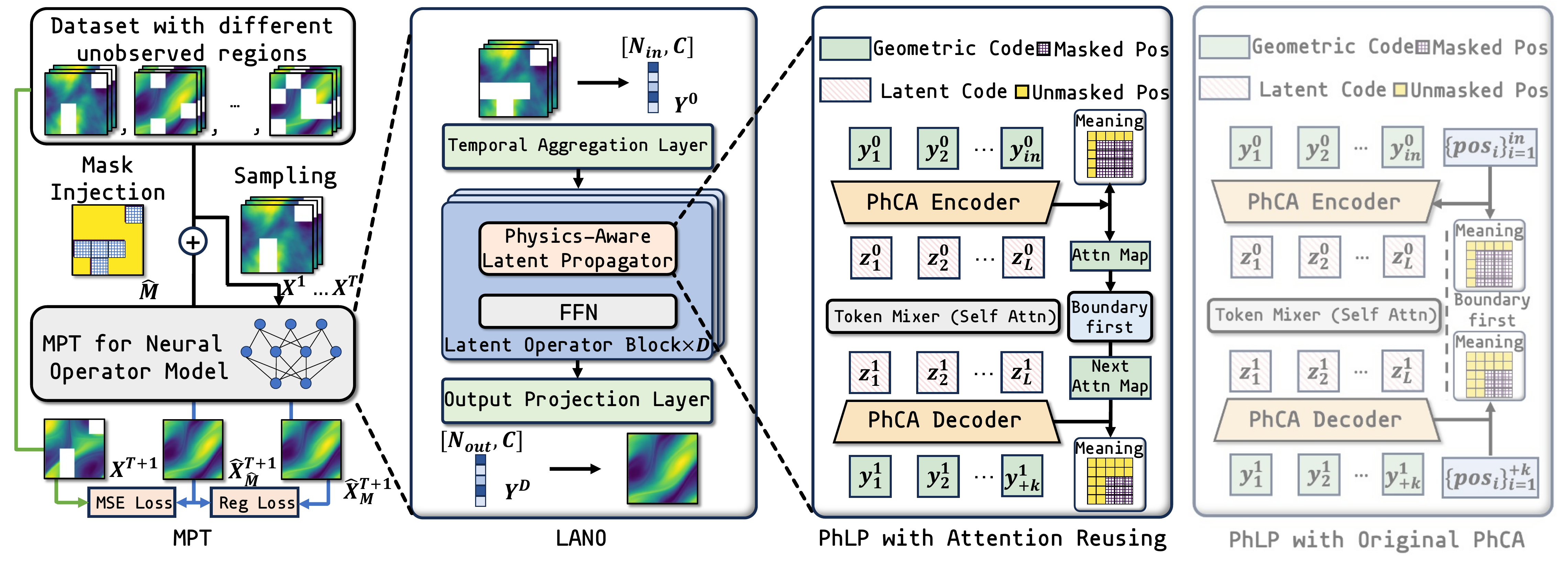}}
    \caption{\textbf{Overview of the proposed \ours model architecture.} We begin by sampling trajectories from partial observations. The model is optimized using the mask-to-predict training strategy~(MPT), where artificially masked past frames are used to predict the next frame. To support this learning objective, we design a novel model architecture incorporating a physics-aware latent propagator~(PhLP), which implements a boundary-first autoregressive framework. 
    }
    \label{fig:model}
\end{figure*}
\section{Proposed Method}
\label{sec:model}


\subsection{Overview of \ours}
\paragraph{Problem setup.} 
We consider time-dependent PDEs defined on a spatial domain $\Omega$, discretized into $N$ points $\mathcal{D} \subset \mathbb{R}^{N \times d}$, where $d$ is the spatial dimensionality.
For a single trajectory, we denote the input and output (solution) measurements as $\boldsymbol{x}_{\theta}(\mathcal{D}_{in})$ and $\boldsymbol{x}_{\theta}(\mathcal{D}_{out})$, where $\mathcal{D}_{in}, \mathcal{D}_{out} \subset \mathcal{D}$ and the data lie in $\mathbb{R}^{N_{in} \times d}$ and $\mathbb{R}^{N_{out} \times d}$, respectively.
In many real-world settings (e.g., weather forecasting), only trajectories of observations are available, while the underlying PDE parameters $\theta$ remain unknown.
In this case, to predict the next timestep, the model must implicitly infer the most likely $\theta$ from adjacent $T$ frames $\{{\boldsymbol{x}^i}\}_{i=1}^T$~\cite{hao2024dpot}.
Our goal is to predict the solution at the next timestep given historical observations. For instance, in fluid dynamics, the input is a sequence of past velocity fields over $\mathcal{D}_{in}$, and the target is the full velocity field over $\mathcal{D}$.
We also define a binary mask $\boldsymbol{M} \in \{0, 1\}^{N}$ to indicate unobserved regions: $M_i = 1$ if position $i$ is observed, and $M_i = 0$ otherwise. 

\paragraph{Objective.} We aim to learn a neural operator $\mathcal{G}_w\left(\boldsymbol{x}^{1:T}\right) = \boldsymbol{x}^{T+1}$ from partially observed temporal PDE datasets, where $w$ denotes learnable parameters. The model auto-regressively takes $T$ frames as input and predicts the next frame. 
We directly supervise the one-step prediction loss using input-output pairs $\{\{\boldsymbol{x}^i\}_{i=1}^{T}(\mathcal{D}_o), \boldsymbol{x}^{T+1}(\mathcal{D}_o)\}_{\# \text{trajectory}}$, where $\mathcal{D}_O \subset \mathcal{D}$ denotes the observed regions with $N_O<N$ points, and each trajectory has a different $\mathcal{D}_o$,
\begin{equation}
    \scalebox{0.98}{$
        \displaystyle
        \min _w \mathcal{L}=\mathbb{E}_{\boldsymbol{x}} \sum_{T \leqslant i < T_{all}}\left\|\mathcal{G}_w\left(\boldsymbol{x}^{i-T+1:i}*M\right)-\boldsymbol{x}^{i+1}\right\|_2^2.
    $}
\end{equation}
We propose a tailored training strategy to address the lack of supervision in unobserved regions. It injects artificial masks $\hat{M}$ as structured noise during training, which helps mitigate cumulative error propagation in autoregressive neural PDE solvers~\cite{brandstettermessage2022} and has been shown to effectively reduce long-term error accumulation in neural operators~\cite{hao2024dpot}.

\subsection{Mask-to-predict training strategy}
To address the lack of supervision in unobserved regions, we propose a mask-to-predict training strategy (MPT) that integrates random input masking into neural operator training. Motivated by the observation that PDE dynamics are often invariant to sparsely sensed fields~\cite{chen2024data}, MPT adapts masking-based learning from NLP and CV~\cite{devlin2019bert, he2022masked} by injecting artificial masks into already incomplete inputs while retaining supervision over observed outputs.

Concretely, we apply a stochastic mask $\hat{M}$ to further occlude input fields, training the model to recover solutions from limited context. This encourages extrapolation beyond available measurements and improves generalization to real-world unobserved regions. We additionally introduce consistency regularization that enforces prediction alignment between original and masked inputs, promoting invariance to artificial sparsity. More details are provided in the Appendix~A.1.



\subsection{Boundary-first autoregressive framework}
Directly predicting all unobserved fields at once often results in blurry and inconsistent outputs, due to the complexity of spatial joint distributions. To mitigate this, we draw inspiration from parallel autoregressive paradigms in visual generation~\cite{wang2025parallelized}, which model local dependencies progressively to improve stability and structure.

Given that PDE solutions are highly sensitive to boundary conditions, we adopt a boundary-first autoregressive framework that propagates information outward from known regions. To enhance efficiency, generation is performed in latent space, which improves both computational cost and spatial consistency. This design allows the model to generate solutions in a physics-aware and coherent manner across partially observed domains.

\subsection{Model Architecture}
\paragraph{Overview.} Existing neural operator architectures lack explicit positional guidance during feature extraction, limiting their predictive ability until the final projection stage. Inspired by LNO~\cite{wang2024latent}, we propose a boundary-first autoregressive framework based on Physics-Cross-Attention~(PhCA). This design introduces spatially aware information during feature propagation, enabling progressive prediction with enhanced physical and positional consistency.

\paragraph{Architecture.} The proposed latent autoregressive neural operator~(\ours) consists of three modules: a temporal aggregation layer for input embedding, latent operator layers that progressively propagate information from observed to unobserved regions guided by PhCA, and an output projection layer for geometric space recovery. The overall architecture is shown in Figure~\ref{fig:model}.


\paragraph{Temporal Aggregation Layer.} This layer embeds input positions and physical measurements from $T$ frames into deep features.
Given positions set $\mathcal{D}=\{s_i\}_{i=1}^{N_{in}}$ with geometric information of $N_{in}$ spatial points and observed fields $\{\boldsymbol{x}^{i}\}_{i=1}^{T}$, we embed them into features $\boldsymbol{Y}^{0}=\{\boldsymbol{y}_i\}_{i=1}^{N_{in}}$ by a linear layer, where $\boldsymbol{y}_i \in \mathbb{R}^{1 \times C}$. The embedding is given by
\begin{equation}
    \boldsymbol{Y}^0=\operatorname{Linear}(\operatorname{Concat}(\mathcal{D}, \boldsymbol{x}^1,\boldsymbol{x}^2,…\boldsymbol{x}^T)).
\end{equation}

\paragraph{Latent Operator Layer.} This layer models PDE operators in the latent space to propagate features from input $\boldsymbol{Y}^0$ to output $\boldsymbol{Y}^D$ using stacked latent operator blocks with a physics-aware latent propagator~(PhLP), 
\begin{equation}
    \begin{aligned}
    & \hat{\boldsymbol{Y}}^l=\text { PhLP }\left(\text { LayerNorm }\left(\boldsymbol{Y}^{l-1}\right)\right)+\boldsymbol{Y}^{l-1} \\
    & \boldsymbol{Y}^l=\text { MLP }\left(\text { LayerNorm }\left(\hat{\boldsymbol{Y}}^l\right)\right)+\hat{\boldsymbol{Y}}^l.
    \end{aligned}
\end{equation}

\paragraph{Output Projection Layer} adopts a linear projection upon the last deep feature $\boldsymbol{Y}^{D}$ and obtains the output as a prediction of solution fields $\boldsymbol{x}^{T+1}$, the output can be calculated through $\boldsymbol{x}^{T+1}=\operatorname{Linear}(\boldsymbol{Y}^D)$.

\begin{table*}[!t]
    \centering
    \setlength{\tabcolsep}{3.8pt}
    \renewcommand{\arraystretch}{0.65}
    \footnotesize
    \textsc{
        \begin{tabular}{ l | c c | c c | c c | c c | c c | c c }
            \toprule
            \multirow{2}{*}{Model} & \multicolumn{4}{c}{ Navier-Stokes } & \multicolumn{4}{c}{ Diffusion-Reaction } & \multicolumn{4}{c}{ ERA5 } \\
            & \multicolumn{2}{c}{point-wise} & \multicolumn{2}{c}{patch-wise} 
            & \multicolumn{2}{c}{point-wise} & \multicolumn{2}{c}{patch-wise} 
            & \multicolumn{2}{c}{point-wise} & \multicolumn{2}{c}{patch-wise} \\
            \midrule
            Train
            & \multicolumn{12}{c}{ Missing ratio $s = 5\%$ } \\
            Test
            & 5\% & 25\% & 5\% & 25\%
            & 5\% & 25\% & 5\% & 25\%
            & 5\% & 25\% & 5\% & 25\% \\
            \midrule
            MIONet~\citeyearpar{jin2022mionet} \cellcolor{white}
            & 0.5662 & / & 0.6531 & /
            & 0.9244 & / & 0.9196 & / 
            & 0.0487 & / & 0.0532 & / \\  
            OFormer~\citeyearpar{li2022transformer}
            & 0.2020 & 0.2082 
            & 0.2021 & 0.2090
            & 0.0301 & 0.0380 
            & 0.0334 & 0.0540 
            & 0.0332 & 0.0334     
            & 0.0289 & 0.0290 \\
            CORAL~\citeyearpar{serrano2023operator}
            & 0.2320 & 0.2510 
            & 0.2479 & 0.2515
            & 0.4758 & / 
            & 0.4916 & /
            & /      & / 
            & /      & /    \\
            GNOT~\citeyearpar{hao2023gnot}
            & 0.2311 & 0.2450
            & 0.2574 & 0.2733
            & 0.9166 & 0.9166
            & 0.9235 & 0.9236
            & 0.0256 & 0.0256
            & 0.0255 & 0.0256 \\
            IPOT~\citeyearpar{lee2024inducing} \cellcolor{white}
            & 0.2528 & 0.2526 & 0.2539 & 0.2556
            & \underline{0.0230} & \underline{0.0284} 
            & \underline{0.0319} & \underline{0.0527}
            & 0.0453             & 0.0458                 
            & 0.0453             & 0.0453   \\
            LNO~\citeyearpar{wang2024latent}
            & \underline{0.1687} & \underline{0.1745}
            & \underline{0.1798} & \underline{0.1879}
            & / & / 
            & / & /
            & \underline{0.0212} & \underline{0.0213}
            & \underline{0.0217} & \underline{0.0216} \\
            \ours-S(Ours)
            & 0.1649 & 0.1645
            & 0.1621 & 0.1694
            & 0.2081 & 0.4596
            & 0.2344 & 0.4719
            & 0.0168 & 0.0170
            & 0.0157 & 0.0154 \\
            \ours(Ours)
            & \textbf{0.1268} & \textbf{0.1275} 
            & \textbf{0.1244} & \textbf{0.1310}
            & \textbf{0.0080} & \textbf{0.0089} 
            & \textbf{0.0148} & \textbf{0.0281}
            & \textbf{0.0122} & \textbf{0.0123}    
            & \textbf{0.0118} & \textbf{0.0120} \\
            \midrule
            \textit{Promotion} 
            & 24.8$\%$ & 26.9$\%$ 
            & 30.8$\%$ & 30.3$\%$
            & 65.2$\%$ & 68.7$\%$ 
            & 53.6$\%$ & 46.7$\%$ 
            & 42.5$\%$ & 42.3$\%$         
            & 45.6$\%$ & 44.4$\%$ \\
            \midrule
            Train
            & \multicolumn{12}{c}{ Missing ratio $s = 25\%$ } \\
            Test
            & 25\% & 50\% & 25\% & 50\%
            & 25\% & 50\% & 25\% & 50\% 
            & 25\% & 50\% & 25\% & 50\% \\
            \midrule
            MIONet~\citeyearpar{jin2022mionet}
            & 0.5703 & / 
            & 0.7312 & /
            & 0.9214 & / 
            & 0.9262 & / 
            & 0.0309 & /  
            & 0.0460 & / \\   
            OFormer~\citeyearpar{li2022transformer}
            & 0.2079 & 0.2159
            & 0.2083 & 0.2226
            & 0.0457 & 0.0670 
            & 0.0818 & 0.1569 
            & 0.0299 & 0.0302
            & 0.0284 & 0.0288 \\
            CORAL~\citeyearpar{serrano2023operator} \cellcolor{white}
            & 0.2264 & 0.2322 
            & 0.2480 & 0.2640
            & 0.4796 & /
            & 0.5366 & / 
            & /      & /
            & /      & /       \\
            GNOT~\citeyearpar{hao2023gnot}
            & 0.2369 & 0.2527
            & 0.2734 & 0.3009
            & 0.9186 & 0.9186
            & 0.9249 & 0.9255
            & 0.0257 & 0.0257
            & 0.0258 & 0.0258 \\
            IPOT~\citeyearpar{lee2024inducing}
            & 0.2568 & 0.2608 
            & 0.2556 & 0.2625
            & \underline{0.0262} & \underline{0.0364} 
            & \underline{0.0485} & \underline{0.1181} 
            & 0.0452             & 0.0455  
            & 0.0453             & 0.0453  \\
            LNO~\citeyearpar{wang2024latent}
            & \underline{0.1732} & \underline{0.1797}
            & \underline{0.1915} & \underline{0.2110}
            & / & /
            & / & /
            & \underline{0.0213} & \underline{0.0214}
            & \underline{0.0214} & \underline{0.0224} \\
            \ours-S(Ours)
            & 0.1632 & 0.1687
            & 0.1749 & 0.1935
            & 0.4600 & 0.6504
            & 0.4718 & 0.6931
            & 0.0164 & 0.0164/
            & 0.0163 & 0.0167 \\
            \ours(Ours) \cellcolor{white}
            & \textbf{0.1274} & \textbf{0.1310} 
            & \textbf{0.1435} & \textbf{0.1608}
            & \textbf{0.0092} & \textbf{0.0127} 
            & \textbf{0.0275} & \textbf{0.0756}
            & \textbf{0.0121} & \textbf{0.0124}
            & \textbf{0.0120} & \textbf{0.0124} \\
            \midrule
            \textit{Promotion}
            & 26.4$\%$ & 27.1$\%$ 
            & 25.1$\%$ & 23.8$\%$
            & 64.9$\%$ & 65.1$\%$ 
            & 43.3$\%$ & 36.0$\%$ 
            & 43.2$\%$ & 42.1$\%$         
            & 43.9$\%$ & 44.6$\%$ \\
            \midrule
            Train
            & \multicolumn{12}{c}{ Missing ratio $s = 50\%$ } \\
            Test
            & 50\% & 75\% 
            & 50\% & 75\%
            & 50\% & 75\% 
            & 50\% & 75\% 
            & 50\% & 75\% 
            & 50\% & 75\% \\
            \midrule
            MIONet~\citeyearpar{jin2022mionet}
            & 0.5754 & / 
            & 0.8382 & /
            & 0.9651 & / 
            & 0.9264 & / 
            & 0.0402 & /  
            & 0.0684 & / \\  
            OFormer~\citeyearpar{li2022transformer}
            & 0.2151 & 0.2218
            & 0.2348 & \underline{0.2597}
            & 0.0972 & 0.1515 
            & 0.2843 & 0.4893 
            & 0.0283 & 0.0288       
            & 0.0273 & 0.0274 \\
            CORAL~\citeyearpar{serrano2023operator}
            & 0.2396 & 0.2662 
            & 0.2835 & 0.3414
            & 0.4909 & / 
            & 0.6225 & / 
            & /      & /   
            & /      & /     \\
            GNOT~\citeyearpar{hao2023gnot}
            & 0.2807 & 0.3051
            & 0.3210 & 0.3609
            & 0.9138 & 0.9145
            & 0.9267 & 0.9275
            & 0.0256 & 0.0258
            & 0.0320 & 0.0344 \\
            IPOT~\citeyearpar{lee2024inducing}
            & 0.2594 & 0.2645 
            & 0.2731 & 0.2878
            & \underline{0.0405} & \underline{0.0678} 
            & \underline{0.1212} & \underline{0.3321} 
            & 0.0437             & 0.0438
            & 0.0392             & 0.0389  \\
            LNO~\citeyearpar{wang2024latent}
            & \underline{0.1749} & \underline{0.1921}
            & \underline{0.2320} & 0.3434
            & / & /
            & / & /
            & \underline{0.0219} & \underline{0.0220}
            & \underline{0.0223} & \underline{0.0276} \\
            \ours-S(Ours)
            & 0.1601 & 0.1775 
            & 0.2397 & 0.3236
            & 0.6583 & 0.8023
            & 0.6526 & 0.7995
            & 0.0158 & 0.0164
            & 0.0168 & 0.0177 \\
            \ours(Ours)
            & \textbf{0.1437} & \textbf{0.1571} 
            & \textbf{0.1835} & \textbf{0.2405}
            & \textbf{0.0128} & \textbf{0.0498} 
            & \textbf{0.0934} & \textbf{0.3249} 
            & \textbf{0.0120} & \textbf{0.0125}
            & \textbf{0.0121} & \textbf{0.0133}   \\
            \midrule
            \textit{Promotion}
            & 17.8$\%$ & 18.2$\%$ & 20.9$\%$ & 7.4$\%$
            & 68.4$\%$ & 26.5$\%$ & 22.9$\%$ & 2.2$\%$ 
            & 45.2$\%$ & 43.2$\%$ & 45.7$\%$ & 51.8$\%$ \\
            \midrule
        \end{tabular}
    }
    \caption{
         Performance comparison on POBench-PDE under various task settings. Relative L2 error is reported (lower is better). The best results are highlighted in bold, and the second best are underlined. Promotion refers to the relative error reduction w.r.t. the second best model ($1-\frac{\text { Our error }}{\text { The second best error }})$. ``/'' denotes that the baseline model is not applicable.
    }
    \label{table:main_result}
\end{table*}

\subsection{Physics-Aware Latent Propagator}
To handle feature interactions across partially observed inputs, we introduce the Physics-Aware Latent Propagator (PhLP), which employs a boundary-first autoregressive framework to progressively reconstruct unobserved regions through physics-guided information propagation.

\paragraph{Physics-Cross-Attention Mechanism.} 
PhLP operates through a Physics-Cross-Attention (PhCA) encoder-decoder in latent space. Given features $\boldsymbol{Y}^{l-1} \in \mathbb{R}^{N_{in} \times C}$ and observation mask $\boldsymbol{M} \in \{0,1\}^{N_{in}}$, the encoder generates latent tokens $\boldsymbol{Z} \in \mathbb{R}^{H \times L \times C_h}$ through:

\begin{equation}
\begin{aligned}
\boldsymbol{S} &= \text{softmax}\left(\frac{\text{MLP}(\boldsymbol{Y}_h)}{\tau}\right) \odot \boldsymbol{M} \\
\boldsymbol{Z} &= \frac{\boldsymbol{S}^T \boldsymbol{Y}_h}{\|\boldsymbol{S}\|_1 + \epsilon},
\end{aligned}
\end{equation}
where $\boldsymbol{S} \in \mathbb{R}^{H \times N_{in} \times L}$ denotes attention maps aggregating spatial information into $L$ latent tokens and $\boldsymbol{Y}_h \in \mathbb{R}^{H \times N_{in} \times C_h}$ is obtained by grouping the input features $\boldsymbol{Y}^{l-1}$ along the channel dimension into $H$ heads, each of dimension $C_h = C / H$. Notably, we directly utilize deep features as keys since they inherently encode spatial position information, eliminating the need for explicit positional encodings.

\paragraph{Boundary-First Propagation.} 
PhLP employs partial convolution~(PConv)~\cite{liu2018image} to propagate information from observed to unobserved regions:
\begin{equation}
\boldsymbol{S}^{next}, \boldsymbol{M}^{next} = \text{PConv}(\boldsymbol{S}, \boldsymbol{M}).
\end{equation}

The updated tokens are processed through a self-attention before decoding:
\begin{equation}
\text { PhLP }(\boldsymbol{Y}^l) = \text{PhCA-Decoder}(\text{Attn}(\boldsymbol{Z}), \boldsymbol{S}^{next}).
\end{equation}
For further details on PhLP, please refer to the Appendix~A.2.

\paragraph{Theoretical Foundation.} 
The solution of PDEs can be formulated as an iterative update process, with existing neural operator approaches leveraging Monte-Carlo sampling to approximate integral operators over the spatial domain $\mathcal{D}$ for each update step (Li et al., 2020; Kovachki et al., 2023). To establish theoretical grounding for our method, we demonstrate that PhLP maintains equivalence to learnable integral operators defined on $\mathcal{D}$.

\begin{theorem}[PhLP as an Integral Operator with Self-Update (Simplified)]
\label{thm:phlp-io}
Let $\mathcal{D} \subset \mathbb{R}^d$ be the computational domain and $\mathcal{D}_{o}^l \subseteq \mathcal{D}$ the observed region at layer $l$. Given features $\mathbf{Y}^{l-1}: \mathcal{D} \to \mathbb{R}^C$, PhLP is equivalent to applying the learnable integral operator
\begin{equation}
\label{eq:phlp-io}
\mathcal{G}^l(\mathbf{Y}^{l-1})(\mathbf{x}^*) = \int_{\mathcal{D}_{o}^l} \kappa^l(\mathbf{x}^*, \boldsymbol{\xi}) \mathbf{Y}^{l-1}(\boldsymbol{\xi}) d\boldsymbol{\xi}, \quad \forall \mathbf{x}^* \in \mathcal{D}
\end{equation}
where the kernel $\kappa^l(\mathbf{x}^*, \boldsymbol{\xi}) \approx \sum_{h=1}^H \sum_{k=1}^L \phi^l_{hk}(\mathbf{x}^*) \psi^l_{hk}(\boldsymbol{\xi})$ admits a low-rank factorization. For observed points $\mathbf{x}^* \in \mathcal{D}_{o}^l$, the kernel includes an identity contribution enabling residual self-update.
\end{theorem}

\begin{proof}[Proof Sketch]
The proof proceeds through three key steps: (i) PhCA encoder produces attention weights $\mathbf{S} \in \mathbb{R}^{H \times N_{o} \times L}$ that define $\psi^l_{hk}(\boldsymbol{\xi}) := \mathbf{S}_{hk}(\boldsymbol{\xi})$; (ii) Decoder parameters provide $\phi^l_{hk}(\mathbf{x}^*)$, realizing the low-rank kernel through attention contraction $\mathbf{S}^T \mathbf{Y}$; (iii) The token mixer enriches the kernel through learnable token transformations while maintaining the operator form. See Appendix~B for complete proof.
\end{proof}

\paragraph{Variant: LANO-S.} 
We investigate a variant (LANO-S) that uses explicit positional encodings to recalculate attention maps during decoding, following LNO's approach and providing a direct comparison between implicit and explicit spatial representations.

\section{Experiment}
\subsection{POBench-PDE}
POBench-PDE is a benchmark suite for PDE-governed tasks under partial observation. 
It contains six settings across three representative PDE tasks. 
~POBench-PDE is specifically constructed to investigate the central question: \textit{can masked prediction tasks within observed regions implicitly guide operator networks to accurately infer unobserved regions?}

\paragraph{Benchmark Construction.}
POBench-PDE comprises benchmark PDE solvers and real-world applications, with all datasets reformulated for partially observed trajectories using temporally consistent masks.
POBench-PDE includes:
\begin{itemize*}
    \item \textbf{Navier–Stokes.} 2D turbulent flow~\cite{li2020fourier}.
    \item \textbf{Diffusion-Reaction.} Biological pattern formation~\cite{takamoto2022pdebench}.
    \item \textbf{ERA5.} Real-world climate data~\cite{hersbach2020era5}.
\end{itemize*}

We exclude tasks like Elasticity that require globally complete inputs, as they do not support meaningful partial observation training. See Appendix~C for more details.

\begin{figure*}[!t]
    \centerline{\includegraphics[width=0.98\linewidth]{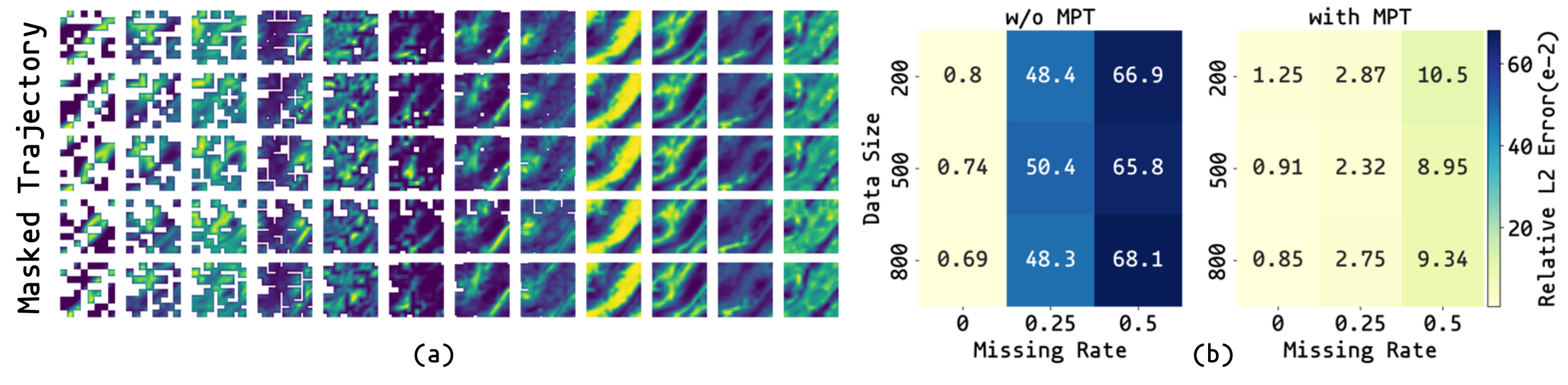}}
    \caption{(a) Visualization of boundary-first latent propagation in PhLP. Each row shows different masked trajectory variants of the same historical sequence. From left to right: features evolve from shallow to deep layers, demonstrating progressive reconstruction from partial observations to coherent physical representations. See Appendix~D.2 for more visualizations. (b) Effectiveness of Mask-to-Predict Training strategy~(MPT) on Diffusion-Reaction dataset. Heatmaps show relative L2 error across different data sizes and missing rates under patch-wise missingness. Left: without MPT, Right: with MPT. Lower is better.}
    \label{fig:visualization}
\end{figure*}

\paragraph{Task.}
To simulate realistic scenarios of partial observation, POBench-PDE supports controlled masking strategies during training and evaluation:
\begin{itemize*}
    \item \textbf{Generalization under Varying Missing Rates.} We simulate varying levels of observational sparsity during training, while keeping the inference stage fixed to full-domain. Specifically, we subsample training data on subsets $\mathcal{D}_o \subset \mathcal{D}$ at different sampling ratios: $s \in\{5 \%, 25 \%, 50 \%\}$.
    \item \textbf{Generalization to Missing Patterns.} We evaluate model robustness under two typical missing patterns: (i) point-wise missingness: Independent Bernoulli sampling of observation points. (ii) patch-wise missingness: Structured masking of contiguous spatial blocks to simulate real-world occlusion or sensor failure. 
\end{itemize*}

\paragraph{Evaluation.}
All evaluations are performed under partially observed input conditions. However, we compute relative \( \ell_2 \) error over the entire ground truth output, rather than restricting it to observed regions. This stands in contrast to masked loss settings and highlights the model’s extrapolative generalization across unobserved regions.


\subsection{Experimental setting}
\paragraph{Baselines.} We benchmark against representative neural operators on POBench-PDE:
\begin{itemize*}
  \item \textbf{MIONet}~\citep{jin2022mionet}: Multi-input DeepONet with autoregressive architecture for interpolation-free predictions.
  \item \textbf{OFormer}~\citep{su2024roformer}: Transformer-based operator with flexible input/output handling.
  \item \textbf{CORAL}~\citep{serrano2023operator}: Neural radiance field-inspired mesh-free PDE solver.
  \item \textbf{GNOT}~\citep{hao2023gnot}: Neural operator utilizing linear Transformers.
  \item \textbf{IPOT}~\citep{lee2024inducing}: Attention-based operator with compressed latent space.
  \item \textbf{LNO}~\citep{wang2024latent}: SOTA neural operator (2024).
\end{itemize*}
These models decode solution fields in a single step but often yield blurry and incoherent results due to challenges in capturing spatial dependencies. In contrast, our boundary-first autoregressive framework progressively incorporates spatial structure, improving prediction accuracy.

\paragraph{Implementation.}
For fair comparison, all models are trained by the AdamW optimizer~\citep{loshchilov2017decoupled} with an initial learning rate of $10^{-3}$, scheduled by OneCycleLR~\citep{smith2018disciplined}. All baseline models follow their official or widely adopted default configurations. Following standard practice, our model employs 8 layers unless otherwise stated. Notably, all models are trained with the MPT strategy, without which none can converge under partial observation. On the two synthetic datasets, we adopt a default patch-wise missingness with a patch size of 4, while for ERA5, we use a larger patch size of 6.
See Appendix~D for comprehensive descriptions of implementations.

\begin{table}[t]
	\centering
	\setlength{\tabcolsep}{1pt} 
        \renewcommand{\arraystretch}{0.9} 
        \footnotesize 
        \sc{
            \begin{tabular}{l|c|c|c|c}
                \toprule
                Source & Benchmarks & \#Time & \#Dim & \#Mesh \\
                \midrule
            FNO & Naview-Stokes & 20 & 2D+Time & 4096 \\
                  \midrule
                  PDE Bench & Diffusion Reaction & 20 & 2D+Time & 4096 \\	     
                \midrule
                  ECMWF & ERA5 & 14 & 2D+Time & 16200 \\
                \bottomrule
            \end{tabular}
        }
        \caption{Summary of POBench-PDE.}
        \label{tab:dataset_summary}
\end{table}
\begin{table}[t]
    \centering
    \setlength{\tabcolsep}{5pt} 
    \renewcommand{\arraystretch}{0.8} 
    \footnotesize 
    \sc{
        \begin{tabular}{c|c|cc|cc}
            \toprule
            \multicolumn{2}{c|}{\multirow{2}{*}{Ablations}} & \#Mem & \#Time & \multicolumn{2}{c}{Relative L2 $\downarrow$} \\
            \multicolumn{2}{c|}{} & (MB) & (s/epoch) & NS & DR \\
            \midrule
            \multirow{5}{*}{\makecell{\#Feats}} 
                & 1  & 42.35  & 261.17 & 0.5254 & 0.4613 \\
                & 8  & 47.23  & 279.00 & 0.2244 & 0.0071 \\
                & 16 & 57.76  & 281.64 & 0.1516 & 0.0073 \\
                & 32 & 95.40  & 282.03 & 0.1274 & 0.0092 \\
                & 64 & 235.38 & 424.74 & 0.1198 & 0.0603 \\
            \midrule
            \multirow{3}{*}{w/o} 
                & BF  & 54.47 & 242.94 & 0.1963 & 0.4649 \\
                & TM  & 95.18 & 248.34 & 0.1459 & 0.0086 \\
                & MPT & 95.40 & 244.50 & 0.4958 & 0.4722 \\
            \midrule
            \multirow{2}{*}{\makecell{Token \\ Mixer}}
                & MLP  & 95.73 & 255.16 & 0.1341 & 0.0089 \\
                & Attn & 95.40 & 282.03 & 0.1274 & 0.0092 \\
            \bottomrule
        \end{tabular} 
        \caption{Ablations on latent features (\#Feats), core components (Boundary-First BF, Token Mixer TM, Mask-to-Predict Training MPT), and token mixer designs. We conduct experiments varying feature numbers, removing components (w/o), and replacing token mixers under a point-wise missingness with 25$\%$ missing rate. Memory usage is calculated with batch size 1. NS: Navier-Stokes, DR: Diffusion-Reaction.}
        \label{table:ablations}
    }
\end{table}

\subsection{Main Results}
\paragraph{Benchmark PDE Solvers.}
As shown in Table~\ref{table:main_result}, \ours achieves near state-of-the-art performance across both benchmarks with all partial observation scenarios. Under patch-wise missingness with less than $50\%$ missing rate, our model consistently outperforms the second-best baseline, with relative error reductions exceeding $17.8\%$.

The reuse-based variant (\ours, default) significantly outperforms the recalculated one (\ours-S), especially on Diffusion-Reaction with high sparsity and structurally disjoint missing regions. We hypothesize that this is due to the instability of decoder-side attention recalculation, which lacks strong guidance and yields incoherent propagation. In contrast, encoder-generated attention—conditioned on aggregated observations—appears to provide stable priors that better guide decoding.

\paragraph{Real-World Applications.}
As shown in Table~\ref{table:main_result}, \ours also achieves superior accuracy compared to the baselines. 
In practical applications such as climate modeling and remote sensing, where missing or partial observations are the norm rather than the exception, the ability of \ours to generalize from incomplete data is crucial.


\begin{figure}[!t]
    \centerline{\includegraphics[width=0.95\linewidth]
    {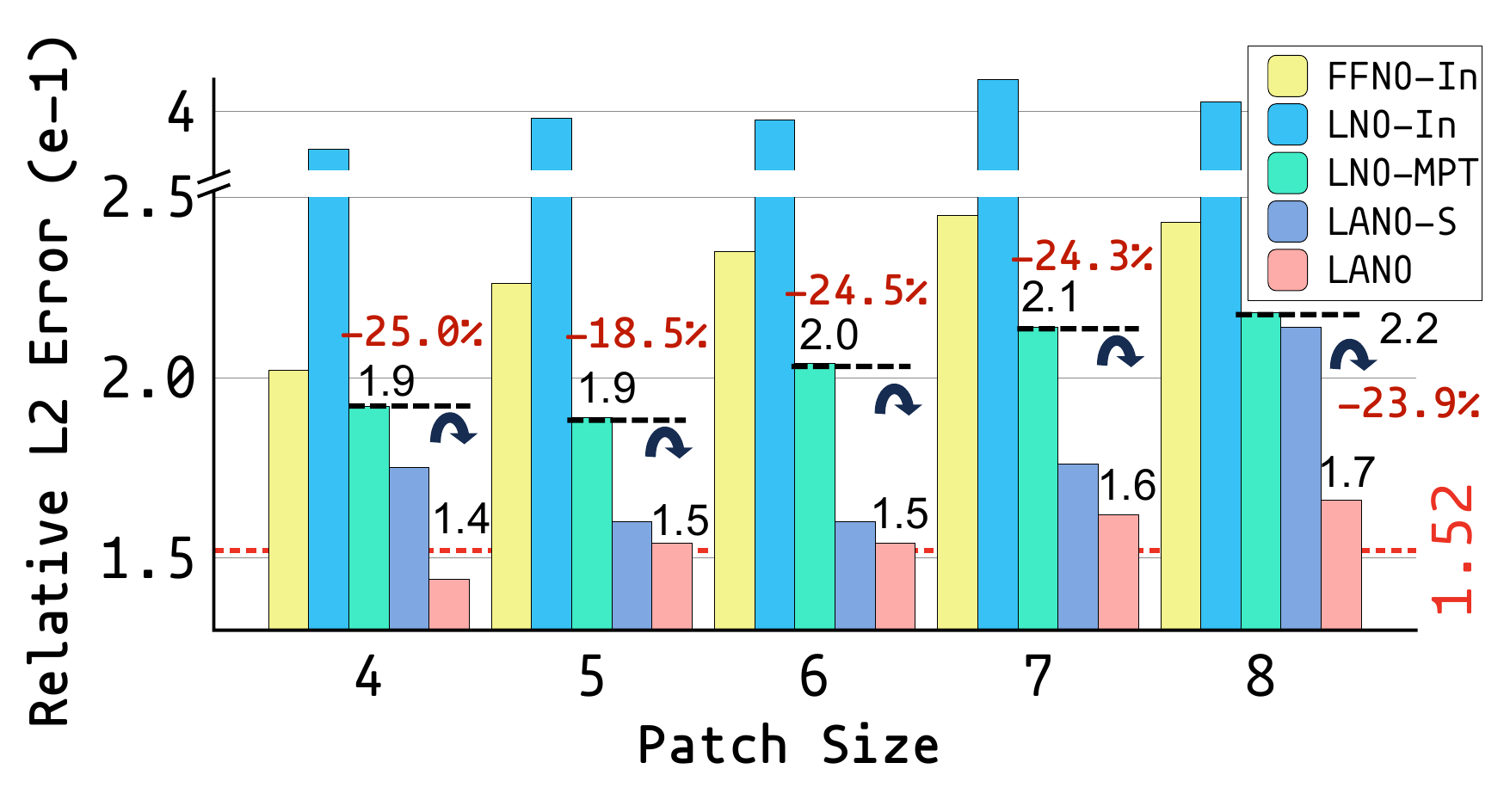}}
    \caption{Performance under varying patch sizes on Navier-Stokes (25\% patch-wise missingness). \ours achieves consistent improvements of 18.5\%-25.0\% over LNO~\citeyearpar{wang2024latent}. Remarkably, our experiments demonstrate that even under 25\% patch-wise missingness, \ours achieves performance comparable to or exceeding that of FNO~\citeyearpar{li2020fourier}, one of the foundational works in neural operator learning.}
    \label{fig:bigPO}
\end{figure}
\subsection{Model Analysis}
\paragraph{Ablations.} We conduct comprehensive ablations on \ours covering feature scaling (\#Feats), component removal (w/o), and token mixer replacement. From Table~\ref{table:ablations}:
\begin{itemize*}
    \item \textbf{Feature scaling:} Optimal at 64 features for NS (77.2\% improvement: 0.5254 $\rightarrow$ 0.1198) and 8 for DR, with larger sizes causing overfitting in simpler PDEs.
    \item \textbf{Component removal:} All components are essential. Removing the boundary-first framework (BF) or mask-to-predict training (MPT) severely degrades performance (289.2$\%$ drop for MPT in NS), while token mixer removal affects complex PDEs more than simple ones.
    \item \textbf{Token mixer replacement:} Attention outperforms MLP for NS (5.0\% improvement), while MLP suffices for simpler DR, indicating task-dependent optimal complexity.
\end{itemize*}

\paragraph{Effectiveness of Mask-to-Predict Training.} We evaluate the effectiveness of MPT on Diffusion-Reaction across varying data sizes (200--800 samples) and missing rates (0$\%$--50$\%$) in a patch-wise missingness setting. 
Figure~\ref{fig:visualization}(b) reveals key observations: 
\begin{itemize*}
    \item \textbf{Significant error reduction:} MPT consistently reduces errors across all configurations, achieving a 84.3$\%$ improvement (0.6688 $\rightarrow$ 0.1051) at a 50$\%$ missing rate with 200 samples. 
    \item \textbf{Robustness to data scarcity:} Without MPT, errors exceed 0.48 under high missing rates. MPT maintains stable performance below 0.11 across most configurations. 
    \item \textbf{Enhanced generalization:} Results confirm that MPT effectively addresses supervision gaps in unobserved regions, enabling robust performance under partial observation.
\end{itemize*}

\paragraph{Visualization of Physics-Aware Latent Propagation.} 
To understand PhLP's processing of partially observed inputs, we visualize intermediate features across layers for the same trajectory under different missing locations. Figure~\ref{fig:visualization}(a) shows progressive evolution from shallow (left) to deep layers (right), revealing key insights: 
\begin{itemize*}
    \item \textbf{Progressive refinement:} Features evolve from sparse fragments to coherent physical structures, confirming boundary-first propagation.
    \item \textbf{Physically consistent reconstruction:} Regardless of missing locations, all input variants converge to similar latent features capturing coherent flow structures, demonstrating PhLP's robustness to partial observation and ability to infer physically meaningful states.
\end{itemize*}

\paragraph{Performance under Varying Patch Sizes.} We evaluate model robustness on Navier-Stokes under patch-wise missingness (25\% missing rate) across patch sizes 4--8. Figure~\ref{fig:bigPO} shows \ours consistently achieves the lowest relative error across all configurations, with 18.5\%--25.0\% improvements over the LNO~\cite{wang2024latent} baseline. While most baselines degrade with larger patches, \ours maintains stable performance (1.44--1.66 range), demonstrating effective handling of contiguous missing regions. For fair comparison, all models are configured with 12 layers when evaluated under patch sizes 5--8.
FFNO~\cite{tran2023factorized}, a scalable variant of FNO, and LNO-Interp use cubic interpolation for data completion before standard training.
MPT-trained models consistently outperform these interpolation-based variants.


\begin{figure}[!t]
    \centerline{\includegraphics[width=\linewidth]{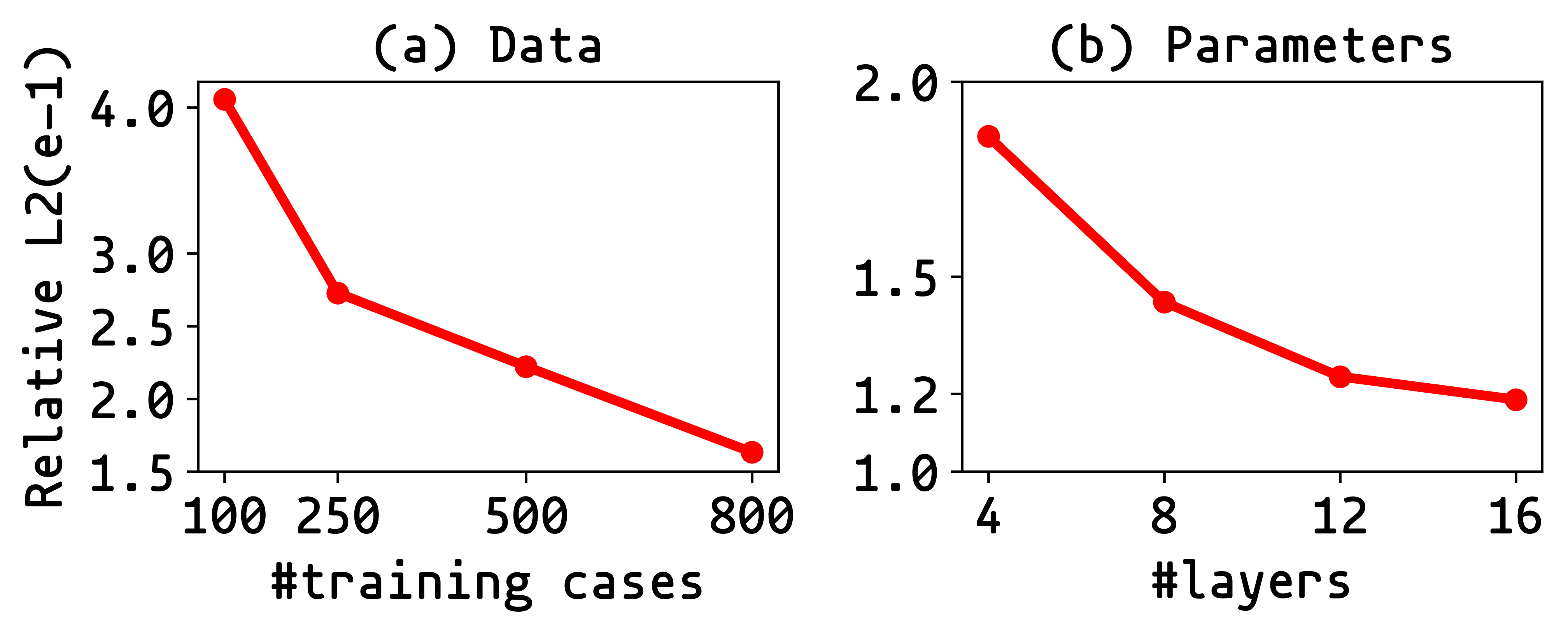}}
    \caption{Evaluation of the model scalability in terms of data size and parameter size. Default setting: 1000 cases, 8 layers.}
    \label{fig:scalability}
\end{figure}
\paragraph{Scalability.} 
We analyse \ours scalability by varying training data volume and model depth on the Navier-Stokes dataset under patch-wise missingness with 25\% missing rate. The results in Figure~\ref{fig:scalability} show consistent performance gains with larger datasets and models, demonstrating its viability as a foundation model backbone for PDE solving.


\section{Conclusion}
In this work, we explore neural operator learning from partial observation, addressing a critical challenge in PDE solver deployment. Through systematic experiments on POBench-PDE, we identify two key insights: (i) Mask-to-predict training enables robust learning by creating pseudo-missing regions, and (ii) Boundary-first autoregressive framework is crucial for progressive reconstruction while maintaining physical consistency. We propose LANO with physics-aware latent propagation~(PhLP), achieving up to 68\% error reduction over existing methods. Our approach demonstrates superior performance on both synthetic benchmarks and real-world climate data.
While our work establishes a foundation for practical PDE solving in data-scarce environments, challenges remain for broader deployment. 
Future work should address these limitations by developing adaptive missing pattern generation strategies and extending the framework to handle irregular geometries and higher-dimensional PDEs.


\section*{Acknowledgments}
This work was supported by the National Key Research and Development Program of China (2024YFE0202900); 
the National Natural Science Foundation of China under Grants 62436001, 62406019, 62536001, and 62176020; 
the Beijing Natural Science Foundation (4244096); 
the Joint Foundation of the Ministry of Education for Innovation Team (8091B042235); 
the State Key Laboratory of Rail Traffic Control and Safety (RCS2023K006); 
and the Talent Fund of Beijing Jiaotong University (2024XKRC075).

    \bibliography{main}

@inproceedings{
gao2025discretizationinvariance,
title={Discretization-invariance? On the Discretization Mismatch Errors in Neural Operators},
author={Wenhan Gao and Ruichen Xu and Yuefan Deng and Yi Liu},
booktitle={The Thirteenth International Conference on Learning Representations},
year={2025},
url={https://openreview.net/forum?id=J9FgrqOOni}
}

@article{gao_pinn,
title = {Active learning based sampling for high-dimensional nonlinear partial differential equations},
journal = {Journal of Computational Physics},
volume = {475},
pages = {111848},
year = {2023},
issn = {0021-9991},
doi = {https://doi.org/10.1016/j.jcp.2022.111848},
url = {https://www.sciencedirect.com/science/article/pii/S0021999122009111},
author = {Wenhan Gao and Chunmei Wang},
}

@inproceedings{wei2025mecot,
  title={MECoT: Markov emotional chain-of-thought for personality-consistent role-playing},
  author={Wei, Yangbo and Huang, Zhen and Zhao, Fangzhou and Feng, Qi and Xing, Wei W},
  booktitle={Findings of the Association for Computational Linguistics: ACL 2025},
  pages={8297--8314},
  year={2025}
}

@article{wei2025vflow,
  title={Vflow: Discovering optimal agentic workflows for verilog generation},
  author={Wei, Yangbo and Huang, Zhen and Li, Huang and Xing, Wei W and Lin, Ting-Jung and He, Lei},
  journal={arXiv preprint arXiv:2504.03723},
  year={2025}
}

@inproceedings{wangaccelerating,
  title={Accelerating Data Generation for Neural Operators via Krylov Subspace Recycling},
  author={Wang, Hong and Hao, Zhongkai and Wang, Jie and Geng, Zijie and Wang, Zhen and Li, Bin and Wu, Feng},
  booktitle={The Twelfth International Conference on Learning Representations},
  year={2024}
}

@inproceedings{wangmixture,
  title={Mixture-of-Experts Operator Transformer for Large-Scale PDE Pre-Training},
  author={Wang, Hong and Xin, Haiyang and Wang, Jie and Yang, Xuanze and Zha, Fei and Jiang, Yan and others},
  booktitle={The Thirty-ninth Annual Conference on Neural Information Processing Systems},
  year={2025}
}

@article{huang2025self,
  title={Self-Attention to Operator Learning-based 3D-IC Thermal Simulation},
  author={Huang, Zhen and Wang, Hong and Yang, Wenkai and Tang, Muxi and Xie, Depeng and Lin, Ting-Jung and Zhang, Yu and Xing, Wei W and He, Lei},
  journal={arXiv preprint arXiv:2510.15968},
  year={2025}
}

@inproceedings{dong2024accelerating,
  title={Accelerating PDE Data Generation via Differential Operator Action in Solution Space},
  author={Dong, Huanshuo and Wang, Hong and Liu, Haoyang and Luo, Jian and Wang, Jie},
  booktitle={International Conference on Machine Learning},
  pages={11395--11411},
  year={2024},
  organization={PMLR}
}

@article{luo2024neural,
  title={Neural Krylov iteration for accelerating linear system solving},
  author={Luo, Jian and Wang, Jie and Wang, Hong and Geng, Zijie and Chen, Hanzhu and Kuang, Yufei and others},
  journal={Advances in Neural Information Processing Systems},
  volume={37},
  pages={128636--128667},
  year={2024}
}

@inproceedings{
wang2025symmap,
title={SymMaP: Improving Computational Efficiency in Linear Solvers through Symbolic Preconditioning},
author={Hong Wang and Jie Wang and Minghao Ma and Haoran Shao and Haoyang Liu},
booktitle={The Thirty-ninth Annual Conference on Neural Information Processing Systems},
year={2025},
url={https://openreview.net/forum?id=Oupeovfx0L}
}

@inproceedings{
wang2025stnet,
title={{STN}et: Spectral Transformation Network for Solving Operator Eigenvalue Problem},
author={Hong Wang and Jiang Yixuan and Jie Wang and Xinyi Li and Jian Luo and huanshuo dong},
booktitle={The Thirty-ninth Annual Conference on Neural Information Processing Systems},
year={2025},
url={https://openreview.net/forum?id=nimTd1IJz1}
}

@inproceedings{lv2025exploiting,
  title={Exploiting Edited Large Language Models as General Scientific Optimizers},
  author={Lv, Qitan and Liu, Tianyu and Wang, Hong},
  booktitle={Proceedings of the 2025 Conference of the Nations of the Americas Chapter of the Association for Computational Linguistics: Human Language Technologies (Volume 1: Long Papers)},
  pages={5212--5237},
  year={2025}
}

@article{gupta2022towards,
  title={Towards Multi-spatiotemporal-scale Generalized PDE Modeling},
  author={Gupta, Jayesh K and Brandstetter, Johannes},
  journal={arXiv preprint arXiv:2209.15616},
  year={2022}
}

@inproceedings{tran2023factorized,
  title     = {Factorized Fourier Neural Operators},
  author    = {Alasdair Tran and Alexander Mathews and Lexing Xie and Cheng Soon Ong},
  booktitle = {The Eleventh International Conference on Learning Representations},
  year      = {2023},
  url       = {https://openreview.net/forum?id=tmIiMPl4IPa}
}

@inproceedings{Cao2021transformer,
  author        = {Shuhao Cao},
  title         = {Choose a Transformer: {F}ourier or {G}alerkin},
  booktitle     = {Advances in Neural Information Processing Systems (NeurIPS 2021)},
  volume        = {34},
  year          = {2021},
  eprint        = {arXiv: 2105.14995},
  primaryclass  = {cs.CL},
  url={https://openreview.net/forum?id=ssohLcmn4-r},
}

@inproceedings{wang2025parallelized,
  title={Parallelized autoregressive visual generation},
  author={Wang, Yuqing and Ren, Shuhuai and Lin, Zhijie and Han, Yujin and Guo, Haoyuan and Yang, Zhenheng and Zou, Difan and Feng, Jiashi and Liu, Xihui},
  booktitle={Proceedings of the Computer Vision and Pattern Recognition Conference},
  pages={12955--12965},
  year={2025}
}

@inproceedings{van2016pixel,
  title={Pixel recurrent neural networks},
  author={Van Den Oord, A{\"a}ron and Kalchbrenner, Nal and Kavukcuoglu, Koray},
  booktitle={International conference on machine learning},
  pages={1747--1756},
  year={2016},
  organization={PMLR}
}

@article{chen2024data,
  title={Data-efficient operator learning via unsupervised pretraining and in-context learning},
  author={Chen, Wuyang and Song, Jialin and Ren, Pu and Subramanian, Shashank and Morozov, Dmitriy and Mahoney, Michael W},
  journal={Advances in Neural Information Processing Systems},
  volume={37},
  pages={6213--6245},
  year={2024}
}

@inproceedings{he2022masked,
  title={Masked autoencoders are scalable vision learners},
  author={He, Kaiming and Chen, Xinlei and Xie, Saining and Li, Yanghao and Doll{\'a}r, Piotr and Girshick, Ross},
  booktitle={Proceedings of the IEEE/CVF conference on computer vision and pattern recognition},
  pages={16000--16009},
  year={2022}
}

@inproceedings{devlin2019bert,
  title={Bert: Pre-training of deep bidirectional transformers for language understanding},
  author={Devlin, Jacob and Chang, Ming-Wei and Lee, Kenton and Toutanova, Kristina},
  booktitle={Proceedings of the 2019 conference of the North American chapter of the association for computational linguistics: human language technologies, volume 1 (long and short papers)},
  pages={4171--4186},
  year={2019}
}

@article{alkin2024universal,
  title={Universal physics transformers: A framework for efficiently scaling neural operators},
  author={Alkin, Benedikt and F{\"u}rst, Andreas and Schmid, Simon and Gruber, Lukas and Holzleitner, Markus and Brandstetter, Johannes},
  journal={Advances in Neural Information Processing Systems},
  volume={37},
  pages={25152--25194},
  year={2024}
}

@article{brandstetter2022clifford,
  title={Clifford neural layers for PDE modeling},
  author={Brandstetter, Johannes and Berg, Rianne van den and Welling, Max and Gupta, Jayesh K},
  journal={arXiv preprint arXiv:2209.04934},
  year={2022}
}

@article{sitzmann2020implicit,
  title={Implicit neural representations with periodic activation functions},
  author={Sitzmann, Vincent and Martel, Julien and Bergman, Alexander and Lindell, David and Wetzstein, Gordon},
  journal={Advances in neural information processing systems},
  volume={33},
  pages={7462--7473},
  year={2020}
}

@inproceedings{lee2024inducing,
  title={Inducing Point Operator Transformer: A Flexible and Scalable Architecture for Solving PDEs},
  author={Lee, Seungjun and Oh, Taeil},
  booktitle={Proceedings of the AAAI Conference on Artificial Intelligence},
  volume={38},
  number={1},
  pages={153--161},
  year={2024}
}

@article{kovachki2023neural,
  title={Neural operator: Learning maps between function spaces with applications to pdes},
  author={Kovachki, Nikola and Li, Zongyi and Liu, Burigede and Azizzadenesheli, Kamyar and Bhattacharya, Kaushik and Stuart, Andrew and Anandkumar, Anima},
  journal={Journal of Machine Learning Research},
  volume={24},
  number={89},
  pages={1--97},
  year={2023}
}

@article{li2024physics,
  title={Physics-informed neural operator for learning partial differential equations},
  author={Li, Zongyi and Zheng, Hongkai and Kovachki, Nikola and Jin, David and Chen, Haoxuan and Liu, Burigede and Azizzadenesheli, Kamyar and Anandkumar, Anima},
  journal={ACM/JMS Journal of Data Science},
  volume={1},
  number={3},
  pages={1--27},
  year={2024},
  publisher={ACM New York, NY}
}

@article{karniadakis2021physics,
  title={Physics-informed machine learning},
  author={Karniadakis, George Em and Kevrekidis, Ioannis G and Lu, Lu and Perdikaris, Paris and Wang, Sifan and Yang, Liu},
  journal={Nature Reviews Physics},
  volume={3},
  number={6},
  pages={422--440},
  year={2021},
  publisher={Nature Publishing Group}
}

@article{loshchilov2017decoupled,
  title={Decoupled weight decay regularization},
  author={Loshchilov, I},
  journal={arXiv preprint arXiv:1711.05101},
  year={2017}
}

@article{lu2021learning,
  title={Learning nonlinear operators via DeepONet based on the universal approximation theorem of operators},
  author={Lu, Lu and Jin, Pengzhan and Pang, Guofei and Zhang, Zhongqiang and Karniadakis, George Em},
  journal={Nature machine intelligence},
  volume={3},
  number={3},
  pages={218--229},
  year={2021},
  publisher={Nature Publishing Group UK London}
}

@article{su2024roformer,
  title={Roformer: Enhanced transformer with rotary position embedding},
  author={Su, Jianlin and Ahmed, Murtadha and Lu, Yu and Pan, Shengfeng and Bo, Wen and Liu, Yunfeng},
  journal={Neurocomputing},
  volume={568},
  pages={127063},
  year={2024},
  publisher={Elsevier}
}

@article{li2020fourier,
  title={Fourier neural operator for parametric partial differential equations},
  author={Li, Zongyi and Kovachki, Nikola and Azizzadenesheli, Kamyar and Liu, Burigede and Bhattacharya, Kaushik and Stuart, Andrew and Anandkumar, Anima},
  journal={arXiv preprint arXiv:2010.08895},
  year={2020}
}

@inproceedings{brandstettermessage2022,
  title={Message Passing Neural PDE Solvers},
  author={Brandstetter, Johannes and Worrall, Daniel E and Welling, Max},
  booktitle={International Conference on Learning Representations},
  year={2022},
}

@article{wen2022u,
  title={U-FNO—An enhanced Fourier neural operator-based deep-learning model for multiphase flow},
  author={Wen, Gege and Li, Zongyi and Azizzadenesheli, Kamyar and Anandkumar, Anima and Benson, Sally M},
  journal={Advances in Water Resources},
  volume={163},
  pages={104180},
  year={2022},
  publisher={Elsevier}
}

@article{wu2023solving,
  title={Solving high-dimensional pdes with latent spectral models},
  author={Wu, Haixu and Hu, Tengge and Luo, Huakun and Wang, Jianmin and Long, Mingsheng},
  journal={arXiv preprint arXiv:2301.12664},
  year={2023}
}

@article{wang2024latent,
  title={Latent neural operator for solving forward and inverse pde problems},
  author={Wang, Tian and Wang, Chuang},
  journal={arXiv preprint arXiv:2406.03923},
  year={2024}
}

@article{xiao2023improved,
  title={Improved operator learning by orthogonal attention},
  author={Xiao, Zipeng and Hao, Zhongkai and Lin, Bokai and Deng, Zhijie and Su, Hang},
  journal={arXiv preprint arXiv:2310.12487},
  year={2023}
}

@article{smith2018disciplined,
  title={A disciplined approach to neural network hyper-parameters: Part 1--learning rate, batch size, momentum, and weight decay},
  author={Smith, Leslie N},
  journal={arXiv preprint arXiv:1803.09820},
  year={2018}
}

@book{zachmanoglou1986introduction,
  title={Introduction to partial differential equations with applications},
  author={Zachmanoglou, Eleftherios C and Thoe, Dale W},
  year={1986},
  publisher={Courier Corporation}
}

@article{morice2021updated,
  title={An updated assessment of near-surface temperature change from 1850: The HadCRUT5 data set},
  author={Morice, Colin P and Kennedy, John J and Rayner, Nick A and Winn, JP and Hogan, Emma and Killick, RE and Dunn, RJH and Osborn, TJ and Jones, PD and Simpson, IR},
  journal={Journal of Geophysical Research: Atmospheres},
  volume={126},
  number={3},
  pages={e2019JD032361},
  year={2021},
  publisher={Wiley Online Library}
}

@article{li2022transformer,
  title={Transformer for partial differential equations' operator learning},
  author={Li, Zijie and Meidani, Kazem and Farimani, Amir Barati},
  journal={arXiv preprint arXiv:2205.13671},
  year={2022}
}

@article{azizzadenesheli2024neural,
  title={Neural operators for accelerating scientific simulations and design},
  author={Azizzadenesheli, Kamyar and Kovachki, Nikola and Li, Zongyi and Liu-Schiaffini, Miguel and Kossaifi, Jean and Anandkumar, Anima},
  journal={Nature Reviews Physics},
  pages={1--9},
  year={2024},
  publisher={Nature Publishing Group UK London}
}

@inproceedings{hao2023gnot,
  title={Gnot: A general neural operator transformer for operator learning},
  author={Hao, Zhongkai and Wang, Zhengyi and Su, Hang and Ying, Chengyang and Dong, Yinpeng and Liu, Songming and Cheng, Ze and Song, Jian and Zhu, Jun},
  booktitle={International Conference on Machine Learning},
  pages={12556--12569},
  year={2023},
  organization={PMLR}
}

@article{hersbach2020era5,
  title={The ERA5 global reanalysis},
  author={Hersbach, Hans and Bell, Bill and Berrisford, Paul and Hirahara, Shoji and Hor{\'a}nyi, Andr{\'a}s and Mu{\~n}oz-Sabater, Joaqu{\'\i}n and Nicolas, Julien and Peubey, Carole and Radu, Raluca and Schepers, Dinand and others},
  journal={Quarterly journal of the royal meteorological society},
  volume={146},
  number={730},
  pages={1999--2049},
  year={2020},
  publisher={Wiley Online Library}
}

@inproceedings{liu2018image,
  title={Image inpainting for irregular holes using partial convolutions},
  author={Liu, Guilin and Reda, Fitsum A and Shih, Kevin J and Wang, Ting-Chun and Tao, Andrew and Catanzaro, Bryan},
  booktitle={Proceedings of the European conference on computer vision (ECCV)},
  pages={85--100},
  year={2018}
}

@article{jin2022mionet,
  title={MIONet: Learning multiple-input operators via tensor product},
  author={Jin, Pengzhan and Meng, Shuai and Lu, Lu},
  journal={SIAM Journal on Scientific Computing},
  volume={44},
  number={6},
  pages={A3490--A3514},
  year={2022},
  publisher={SIAM}
}

@article{serrano2023operator,
  title={Operator learning with neural fields: Tackling pdes on general geometries},
  author={Serrano, Louis and Le Boudec, Lise and Kassa{\"\i} Koupa{\"\i}, Armand and Wang, Thomas X and Yin, Yuan and Vittaut, Jean-No{\"e}l and Gallinari, Patrick},
  journal={Advances in Neural Information Processing Systems},
  volume={36},
  pages={70581--70611},
  year={2023}
}

@article{takamoto2022pdebench,
  title={Pdebench: An extensive benchmark for scientific machine learning},
  author={Takamoto, Makoto and Praditia, Timothy and Leiteritz, Raphael and MacKinlay, Daniel and Alesiani, Francesco and Pfl{\"u}ger, Dirk and Niepert, Mathias},
  journal={Advances in Neural Information Processing Systems},
  volume={35},
  pages={1596--1611},
  year={2022}
}

@article{wu2024transolver,
  title={Transolver: A fast transformer solver for pdes on general geometries},
  author={Wu, Haixu and Luo, Huakun and Wang, Haowen and Wang, Jianmin and Long, Mingsheng},
  journal={arXiv preprint arXiv:2402.02366},
  year={2024}
}

@article{hao2024dpot,
  title={Dpot: Auto-regressive denoising operator transformer for large-scale pde pre-training},
  author={Hao, Zhongkai and Su, Chang and Liu, Songming and Berner, Julius and Ying, Chengyang and Su, Hang and Anandkumar, Anima and Song, Jian and Zhu, Jun},
  journal={arXiv preprint arXiv:2403.03542},
  year={2024}
}

@article{puzyrev2019deep,
  title={Deep learning electromagnetic inversion with convolutional neural networks},
  author={Puzyrev, Vladimir},
  journal={Geophysical Journal International},
  volume={218},
  number={2},
  pages={817--832},
  year={2019},
  publisher={Oxford University Press}
}

@inproceedings{feng2021task,
  title={Task transformer network for joint MRI reconstruction and super-resolution},
  author={Feng, Chun-Mei and Yan, Yunlu and Fu, Huazhu and Chen, Li and Xu, Yong},
  booktitle={Medical Image Computing and Computer Assisted Intervention--MICCAI 2021: 24th International Conference, Strasbourg, France, September 27--October 1, 2021, Proceedings, Part VI 24},
  pages={307--317},
  year={2021},
  organization={Springer}
}

@article{zheng2019applications,
  title={Applications of supervised deep learning for seismic interpretation and inversion},
  author={Zheng, York and Zhang, Qie and Yusifov, Anar and Shi, Yunzhi},
  journal={The Leading Edge},
  volume={38},
  number={7},
  pages={526--533},
  year={2019},
  publisher={Society of Exploration Geophysicists}
}

@article{li2023fourier,
  title={Fourier neural operator with learned deformations for pdes on general geometries},
  author={Li, Zongyi and Huang, Daniel Zhengyu and Liu, Burigede and Anandkumar, Anima},
  journal={Journal of Machine Learning Research},
  volume={24},
  number={388},
  pages={1--26},
  year={2023}
}

\end{document}